\documentclass{article}

\PassOptionsToPackage{numbers,compress}{natbib}
\usepackage[preprint]{neurips_2026}

\usepackage[utf8]{inputenc}
\usepackage[T1]{fontenc}
\usepackage{hyperref}
\usepackage{url}
\usepackage{booktabs}
\usepackage{array}
\usepackage{multirow}
\usepackage{graphicx}
\usepackage{amsfonts}
\usepackage{amsmath}
\usepackage{amssymb}
\usepackage{amsthm}
\usepackage{mathtools}
\usepackage{microtype}
\usepackage{xcolor}
\usepackage{enumitem}

\theoremstyle{definition}

\newtheorem{assumption}{Assumption}[section]
\newtheorem{lemma}{Lemma}[section]
\newtheorem{proposition}{Proposition}[section]
\newtheorem{theorem}{Theorem}[section]

\newcommand{\bits}{\mathbf{b}}
\newcommand{\loss}{\Delta Q}
\newcommand{\cfgpred}{\widehat{\Delta Q}_{\mathrm{cfg}}}
\newcommand{\reqpred}{\widehat{\Delta Q}_{\mathrm{req}}}
\newcommand{\losshat}{\widehat{\ell}}
\newcommand{\affinity}{\rho}
\newcommand{\experts}{\mathcal{E}}
\newcommand{\instances}{\mathcal{A}}
\newcommand{\classes}{\mathcal{G}}
\newcommand{\FWP}{\mathrm{FWP}}
\newcommand{\Regret}{\mathrm{Regret}}

\title{Quality-Constrained Routing over a Fixed Pool of Quantized Mixture-of-Experts Instances}
\author{%
  Zhenghong Huang,\quad
  Hongfan Wu$^{*}$,\quad
  Jiheng Zhang \\
  \\
  The Hong Kong University of Science and Technology\\
  Clear Water Bay, Kowloon, Hong Kong\\
  \\
  \texttt{\{zhuangdr, hwucn\}@connect.ust.hk,\;\; jiheng@ust.hk}
}

\begin{document}

\maketitle

\begin{abstract}
% S1: setting and empirical phenomenon; opener.
Quantized Mixture-of-Experts (MoE) services can hold several pre-materialized instances of one base model, but quantization damage varies sharply across requests and bitwidths.
% S2: scope boundary; connector names the upstream decision.
Because instance materialization and replica counts consume memory and require slow reconfiguration, we treat them as upstream provisioning decisions and study routing within a fixed resident pool.
% S3: optimization objective; connector names the fixed-pool boundary.
Within this fixed-pool boundary, we route each request to maximize modeled throughput under a class-level expected quality-degradation budget and measured instance capacities.
% S4: operational signal; connector names the request-risk need.
To predict this request-specific risk, we introduce FWP (Fragility-Weighted Perplexity), computed from prompt tokens on a reference-instance prefill and calibrated to candidate-instance degradation.
% S5: theory scope; connector names the FWP mechanism.
Underlying FWP is an exact two-expert affinity--fragility decomposition and a conditional multi-layer top-$k$ expansion whose bias, interaction, route-change, separability, and higher-order terms remain explicit.
% S6: optimization result; connector names calibrated risks.
Using these calibrated risks, a window-level linear program yields a signed reduced-reward score that is KKT-consistent with the LP optimum under optimal prices and primal-feasible tie allocation.
% S7: complete-instance quality evidence; connector names evaluation.
On 88 extended Qwen prompts, complete W2, W3, and W4 instances quantizing all 6,144 expert blocks incur mean $\Delta$NLL of $0.9437$, $0.1832$, and $0.0513$.
% S8: headline attribution; connector names the common population.
Under the same population and $\tau=0.1513$, FWP allocation reaches a $1.284\times$ offline model-based multiplier versus $1.253\times$ for request-agnostic mixing and $1.000\times$ for static W4, an incremental $2.5\%$ relative FWP gain.
\end{abstract}

\section{Introduction}
Sparse MoE models provide high parameter capacity while activating only a subset of experts per token~\citep{shazeer2017outrageously,lepikhin2021gshard,fedus2022switch,jiang2024mixtral,dai2024deepseekmoe,qwen3technical2025}.
Yet serving still pays to keep expert weights resident, which makes post-training expert quantization an important memory and throughput tool~\citep{frantar2023gptq,lin2024awq,badri2023hqq,frantar2024qmoe,kim2023moqe}.
Complete-instance measurements in this paper show that the resulting quality cost is heterogeneous: on the same 88 Qwen requests, mean $\Delta$NLL is $0.9437$, $0.1832$, and $0.0513$ for W2, W3, and W4, while the W3 median is only $0.1387$.
This gap between population means and request-level outcomes creates an allocation opportunity once several pre-materialized instances are already resident.

This paper studies the conditional routing problem opened up by this heterogeneity.
At serving time, a fixed pool $\instances$ of pre-materialized quantized model copies is held resident, each frozen at a quantization configuration of the same MoE base; we call these copies \textit{instances}.
This raises the central question of this paper:
\begin{center}
\textit{How should a serving system route each incoming request among the pool's instances to maximize throughput under a per-class quality-degradation budget?}
\end{center}
The instance configurations and replica counts are upstream provisioning decisions: they determine parameter residence, consume memory, and may require weight loading or placement changes.
During a routing epoch they are immutable, so adding replicas changes the input pool rather than supplying another per-request action.
We formalize this two-timescale boundary in Section~\ref{sec:background}: an outer memory-feasible provisioning problem may choose any replica vector and pay reconfiguration cost, while the contribution of this paper is the inner quality-constrained routing problem conditional on that vector.
This scope does not assert that the evaluated W2/W3/W4 pool is the globally best use of a fixed fleet budget.
It is consistent with production MoE practice: DeepSeek-V3 reports adjusting redundant high-load experts from online statistics periodically, for example every ten minutes, while request dispatch occurs at a much faster timescale~\citep{deepseekv3,deepseekeplb}.

Adaptive routing requires a prompt-time signal correlated with candidate-instance quality loss.
We use \textbf{expert fragility} for the request-independent sensitivity of an expert to quantization~\citep{chowdhury2026efficient} and \textbf{expert affinity} for how strongly a request activates that expert~\citep{shazeer2017outrageously,fedus2022switch}.
The exact two-expert model factors population loss into affinity and intrinsic degradation; the realistic multi-layer top-$k$ bridge is instead a conditional blockwise Taylor expansion with explicit bias, interaction, separability, route-change, and remainder terms.
We operationalize the mechanism through \textbf{FWP (Fragility-Weighted Perplexity)}, a request-local weighted prompt NLL computed on a designated reference instance.
Offline calibration maps this raw prompt-time statistic, together with declared request features such as prefix length, to candidate-instance degradation; it does not use an unavailable full-precision score for the unseen request.
If the selected target differs from the reference instance, the execution path performs and charges a second prefill because KV state is generally not transferable across differently quantized instances.

We convert calibrated risk into allocation through a \textbf{window-level linear program} over the fixed pool.
The LP maximizes served tokens subject to per-instance capacity and per-class expected quality-degradation constraints.
Its dual yields a signed reduced-reward score with one capacity price and one class-quality price.
At optimal dual prices, every route used by a primal optimum maximizes this score; resolving ties according to a primal-optimal allocation makes the score \textbf{KKT-consistent with the LP optimum}.
An arbitrary greedy tie rule guarantees stationarity only, so feasibility and price convergence remain explicit conditions rather than deployment guarantees.

The empirical evidence is intentionally separated by protocol (Section~\ref{sec:results}).
Complete Qwen measurements quantize all 6,144 expert blocks and supply the quality backbone; DeepSeek-V2-Lite remains sampled diagnostic evidence and therefore does not support a complete-instance cross-architecture claim.
On the complete Qwen population at $\tau=0.1513$, static W4 reaches a $1.000\times$ offline model-based multiplier, request-agnostic budget-aware mixing reaches $1.253\times$, FWP reaches $1.284\times$, and a restricted true-loss ordering reaches $1.303\times$.
Thus heterogeneous budget allocation explains $0.253$ of the $0.284$ improvement over static W4, while FWP contributes $0.031$ multiplier points, or $2.5\%$ relative to request-agnostic mixing.
Prompt-only cross-instance FWP evidence is moderate (Spearman $0.586$, AUROC $0.747$), and prefix correlation ranges from $0.238$ at 64 tokens to $0.999$ at 1,024 tokens; these results motivate length-aware abstention rather than a universal deployment guarantee.

Our contributions are:

\begin{enumerate}[leftmargin=*, itemsep=1pt, topsep=3pt]
\item \textbf{A fixed-pool problem formulation with an explicit provisioning boundary.}
We formalize memory-feasible replica selection and reconfiguration cost at a slow timescale, then define quality-constrained request routing over the resulting immutable pool as the paper's decision problem.

\item \textbf{A claim-controlled theory and request-risk signal.}
We separate the exact two-expert affinity--fragility identity, the centered $2^{-2b}$ reduced form, the conditional multi-layer top-$k$ extension, and the within-request covariance condition under which FWP amplifies degradation.
The LP result applies to the signed risk formulation and is KKT-consistent with the LP optimum only under its stated price, tie-allocation, and feasibility conditions.

\item \textbf{Complete-instance evidence with gain attribution.}
For 88 extended Qwen requests and complete W2/W3/W4 instances, we report the corrected $1.284\times$ offline model-based allocation multiplier and show that FWP adds $2.5\%$ relative over the $1.253\times$ request-agnostic feasible mixture.
We separately report prompt-length, calibration, memory, endpoint, and controller evidence without merging them into a live equal-resource throughput claim.
\end{enumerate}
\section{Background and Problem Setting}
\label{sec:background}

\subsection{Fixed-Pool Routing and the Provisioning Boundary}

\paragraph{Notation.}
We consider an autoregressive MoE model with $L$ sparse layers. At token $t$, layer $\ell$ routes the hidden state $h_t$ through a top-$k$ router with gate weights $g_{t,e}$ to an expert subset $\mathcal{K}_t \subseteq \experts_\ell$, producing $y_t = \sum_{e \in \mathcal{K}_t} g_{t,e} f_e(h_t)$. Only the expert weights $(W_{1,e}, W_{2,e})$ are quantized; the router stays at reference precision, and a quantization configuration assigns bitwidth $b_e$ to each expert $e$. This asymmetry is the architectural lever: routing is observable in real time, while quantization damage concentrates on a small, request-dependent subset of experts. For a request $x$ with token set $\mathcal{T}(x)$, per-expert affinity is
\begin{equation}
\affinity_e(x) = \frac{1}{|\mathcal{T}(x)|} \sum_{t \in \mathcal{T}(x)} g_{t,e},
\label{eq:affinity-def}
\end{equation}
turning a per-expert quantization choice into a per-request quality risk.

\paragraph{Two decision timescales.}
Let $\overline{\instances}$ be a catalog of quantization configurations that could be materialized, and let provisioning epoch $q$ last $H_q$ routing windows.
The integer $y_a^{(q)}\ge 0$ records how many resident replicas of catalog configuration $a$ are available during epoch $q$.
If $m_a$ is its parameter-memory footprint, $M_{\mathrm{sh}}$ is memory shared across the service, and $\overline M$ is the provisioning budget, the feasible replica vectors satisfy
\begin{equation}
\mathcal Y
=
\left\{y\in\mathbb Z_+^{|\overline{\instances}|}:
M_{\mathrm{sh}}+\sum_{a\in\overline{\instances}}m_a y_a\le \overline M\right\}.
\label{eq:memory-feasible-provisioning}
\end{equation}
The resident endpoint pool is
\begin{equation}
\instances\!\left(y^{(q)}\right)
=
\left\{(a,r):a\in\overline{\instances},\ r=1,\ldots,y_a^{(q)}\right\},
\label{eq:resident-pool}
\end{equation}
so identical replicas are distinct capacity endpoints even though they share a quantization configuration.

\paragraph{Reconfiguration cost.}
Changing the resident vector may require loading, evicting, or relocating weights.
We represent this slower cost as
\begin{equation}
C_{\mathrm{reconf}}\!\left(y^{(q)},y^{(q-1)}\right)
=
\sum_{a\in\overline{\instances}}
\left[
c_a^{\mathrm{on}}\bigl(y_a^{(q)}-y_a^{(q-1)}\bigr)_+
+c_a^{\mathrm{off}}\bigl(y_a^{(q-1)}-y_a^{(q)}\bigr)_+
\right].
\label{eq:reconfiguration-cost}
\end{equation}
An upstream provisioner could therefore compare replica fleets through the outer problem
\begin{equation}
\max_{y^{(q)}\in\mathcal Y}
\left\{
V_q^{\star}\!\left(\instances(y^{(q)})\right)
-\frac{C_{\mathrm{reconf}}(y^{(q)},y^{(q-1)})}{H_q}
\right\},
\label{eq:outer-provisioning-boundary}
\end{equation}
where $V_q^{\star}(\cdot)$ is the optimal value of the inner routing problem.
Equation~\eqref{eq:outer-provisioning-boundary} defines the boundary but is not optimized by our algorithm: throughout a routing epoch, $y^{(q)}$ and hence $\instances=\instances(y^{(q)})$ are fixed.

\paragraph{Routing problem.}
Requests arrive with class $g_t$, expected output-token count $L_{g_t}$, profiled service cost, and quality tolerance $\tau_{g_t}$.
A policy maps each complete request to one resident endpoint $a_t\in\instances$; it does not quantize, load, evict, or create replicas online.
The objective is to maximize served tokens subject to per-instance capacity and per-class expected quality degradation, formalized by the window-level LP in Section~\ref{sec:theory-lp}.
This contract makes the replica comparison precise: a different replica vector changes the outer input, whereas the contribution evaluated here is allocation conditional on the same fixed pool.

\paragraph{Operational precedent and limitation.}
DeepSeek-V3 provides public evidence that MoE placement can run at a slower cadence than request dispatch: its deployment duplicates high-load experts, detects them from online statistics, and adjusts the redundant set periodically, for example every ten minutes~\citep{deepseekv3}.
The official EPLB implementation computes expert replication and GPU-placement plans from estimated loads~\citep{deepseekeplb}.
This evidence supports a two-timescale abstraction for expert placement; it does not measure whole-instance loading cost, prove that ten minutes is optimal, or establish the memory feasibility of our W2/W3/W4 pool.

\subsection{Related Work}
\label{sec:related}

Our work sits at the junction of four lines, with one ancillary connection. First, sparse MoE models---from gated experts \citep{shazeer2017outrageously} through GShard and Switch Transformers \citep{lepikhin2021gshard,fedus2022switch} to Mixtral, DeepSeekMoE, and Qwen MoE \citep{jiang2024mixtral,dai2024deepseekmoe,qwen2moe2024,qwen3technical2025}---supply the architectural substrate. Second, post-training quantization \citep{frantar2023gptq,lin2024awq,xiao2023smoothquant,badri2023hqq} and MoE-specific compressors \citep{frantar2024qmoe,kim2023moqe,chowdhury2026efficient} produce candidate instances; the per-expert router-norm fragility indicator of \citet{chowdhury2026efficient} motivates the request-level FWP signal. Third, serving and placement systems \citep{yu2022orca,kwon2023pagedattention,zhong2024distserve,sun2024llumnix,deepseekv3} optimize batching, capacity, or expert placement. We take their provisioned endpoints as input and add a request-level quantization-risk term, so replica selection and routing are complementary rather than competing claims. Fourth, cost-aware routing across heterogeneous models \citep{chen2023frugalgpt,ong2024routellm,hu2024routerbench,ding2024hybridllm} selects different models rather than quantization instances of one base. Selective prediction \citep{guo2017calibration,geifman2017selective} is an ancillary connection because the calibrated object is request-level degradation. Appendix~\ref{sec:related-extended} expands this positioning.

\section{Theoretical Framework}
\label{sec:theory}

This section separates three theoretical objects that the earlier presentation conflated. First, the two-expert sparse-MoE model gives an exact affinity--fragility identity and a centered-noise $2^{-2b}$ expansion. Second, a within-request covariance identity states exactly when the deployed FWP weighting amplifies token degradation; the two-expert top-$1$ model alone cannot prove that property because its within-request weights collapse. Third, the fixed-pool LP yields a signed reduced-reward score whose KKT interpretation requires optimal prices and primal-feasible tie allocation. Appendix~\ref{app:qb-proof} gives the local derivations and a conditional multi-layer top-$k$ bridge with explicit correction terms.

\subsection{Expert Fragility and the Reduced-Form Surrogate}
\label{sec:theory-sensitivity}

We begin with an explicit model that isolates the structure we want to exploit. The full derivation, with all intermediate steps, appears in Appendix~\ref{app:qb-proof}; here we present the key objects and their theoretical guarantees.

\paragraph{Two-expert sparse-MoE model.}
Consider a sparse top-1 MoE classifier with two experts, $e_r$ and $e_c$. Each request belongs to a rare population (mass $\alpha < \frac{1}{4}$) routed to $e_r$, or a common population (mass $1-\alpha$) routed to $e_c$. The router is deterministic: $g(x) = e_r$ if $X=r$, else $e_c$. Let $m_j > 0$ be the full-precision signed classification margin for population $j \in \{r,c\}$, and assume $0 < m_r < m_c$: the rare expert has a weaker activation margin, reflecting the fragility mechanism in the literature~\citep{chowdhury2026efficient}. The prediction loss is logistic, $\ell(z) = \log(1 + e^{-z})$, and full-precision quality is $Q_{\mathrm{full}} := -\mathbb{E}[\ell(m_X)]$.

When each expert $e_j$ is quantized to $b_j$ bits, the margin is perturbed by a zero-mean random variable $\xi_j(b_j)$ with variance scaling $v_j 2^{-2b_j}$ (the standard uniform-quantization scaling). The quantized quality is $Q(\bits) := -\mathbb{E}[\ell(m_X + \xi_X(b_X))]$, and the quality loss is $\loss(\bits) := Q_{\mathrm{full}} - Q(\bits)$.

\paragraph{Exact affinity--fragility decomposition.}
Conditioning on the request type gives the exact decomposition
\begin{equation}
\loss(\bits) = \alpha\,\Delta_r(b_r) + (1-\alpha)\,\Delta_c(b_c),
\label{eq:model-mixture-main}
\end{equation}
where $\Delta_j(b_j) := \mathbb{E}[\ell(m_j + \xi_j(b_j)) - \ell(m_j)]$. This cleanly factors the loss into two objects: the coefficients $\alpha$ and $1-\alpha$ are \emph{affinity} weights (how often each expert is used), while $\Delta_r$ and $\Delta_c$ are \emph{intrinsic fragility} functions (how much damage quantization does to each expert). These two quantities should not be combined, and the routing score we derive in Section~\ref{sec:theory-lp} inherits both.

\paragraph{Small-noise expansion and reduced form.}
To obtain a practical parameterization, we expand each $\Delta_j(b_j)$ around the full-precision margin. Under centered noise with variance proportional to $2^{-2b_j}$, the first-order term vanishes and the leading term is second order (Appendix~\ref{app:qb-proof}, Lemma~\ref{lem:model-local-expansion}). Biased deterministic quantization instead has a separate signed first-order coefficient multiplying $2^{-b_j}$; it is not the same positive fragility coefficient.

\begin{theorem}[Centered-noise reduced form]
\label{thm:model-reduced-form-main}
Under the two-expert model, zero-mean perturbations, and the moment bounds in Appendix~\ref{app:qb-proof},
\begin{equation}
\loss(\bits) = \alpha\, c_r\,2^{-2b_r} + (1-\alpha)\, c_c\,2^{-2b_c} + r(\bits),
\label{eq:model-reduced-form-main}
\end{equation}
where $c_j = \frac{1}{2}\ell''(m_j)v_j > 0$ is an expert-specific fragility coefficient and the remainder satisfies $|r(\bits)| \le \kappa\sum_j 2^{-3b_j}$ for a constant $\kappa$ depending on local curvature and third moments.
\end{theorem}

Theorem~\ref{thm:model-reduced-form-main} justifies the $2^{-2b}$ empirical family under its stated stochastic assumptions. Appendix Lemma~\ref{lem:model-first-order} separately gives the biased $2^{-b}$ branch with a signed coefficient. The fragility coefficient in the centered branch is controlled by a monotone decreasing upper bound in the expert's router-derived margin proxy $\widetilde{\Lambda}_j$ (Appendix Proposition~\ref{cor:model-structured}); deterministic HQQ or RTN checkpoints are empirical tests, not automatic instances of the centered-noise theorem.

\paragraph{Empirical surrogate and request-level lift.}
Guided by the two separate expansions, the configuration-level models fitted to real perturbation data are
\begin{equation}
\cfgpred(\bits) = \sum_{e \in \experts} \widehat{c}^{\mathrm{free}}_e\,\phi(b_e),
\label{eq:reduced-form}
\end{equation}
where $\widehat{c}^{\mathrm{free}}_e$ is calibrated by leave-one-expert-out perturbations and the $2^{-b}$ and $2^{-2b}$ families are fitted and reported as distinct empirical specifications. A restricted structured family uses router norm and MaxVar (Appendix~\ref{app:qb-proof}). These sampled-expert regressions diagnose ranking structure; they are not the complete-instance quality evidence used for the headline allocation result.

The configuration-level form motivates the following request-level surrogate:
\begin{equation}
\reqpred(x, \bits) = \sum_{e \in \experts} \affinity_e(x)\,\widehat{c}_e\,\phi(b_e),
\label{eq:request-lift}
\end{equation}
Here affinity is observable from gates and the fitted fragility coefficient is request-independent. In a realistic multi-layer top-$k$ network, Eq.~\eqref{eq:request-lift} requires an additional separability approximation between exact block sensitivity and mean-gate affinity; the conditional extension in Appendix~\ref{app:qb-proof} retains its residual together with bias, interaction, route-change, and Taylor terms. Affine transfer is therefore an empirical calibration for the fixed evaluated instances, not a theorem for arbitrary expert-wise configurations.

\subsection{FWP: Fragility-Weighted Perplexity}
\label{sec:theory-fwp}

The lifted surrogate in Eq.~\eqref{eq:request-lift} is a metric-agnostic predictor. The routing gain, however, depends on how strongly the quality observable separates fragile-expert degradation from bulk-token noise. We now construct a concrete observable---FWP---and prove that it amplifies the signal of interest.

\paragraph{Design.}
We design FWP from the router-norm fragility signal in prior work~\citep{chowdhury2026efficient}. For the non-padding prompt-token index set $I_p(x)$, let $\bar{\Lambda}_i(x)$ be the mean router norm of experts activated by token $i$. Define weights normalized within the request,
\begin{equation}
w_i(x) = \frac{(\bar{\Lambda}_i(x) + \epsilon_w)^{-1}}{\frac{1}{|I_p(x)|}\sum_{j\in I_p(x)}(\bar{\Lambda}_j(x) + \epsilon_w)^{-1}},
\label{eq:fwp-weights}
\end{equation}
with $\epsilon_w > 0$, and the log-FWP observable
\begin{equation}
\ell_{\FWP}(x; \bits) = \frac{1}{|I_p(x)|}\sum_{i\in I_p(x)} w_i(x)\bigl(-\log p_{\bits}(x_i \mid x_{<i})\bigr).
\label{eq:fwp-def}
\end{equation}
The statistic is request-local: tokens from other requests and padding positions do not enter either sum. For an unseen request, a reference-instance prefill produces the raw value in Eq.~\eqref{eq:fwp-def}; an offline calibration map uses that value, prompt length, and declared class features to predict candidate-instance degradation. No full-precision score for the same unseen request is available online.

\paragraph{Within-request amplification identity.}
For an offline matched pair of a quantized and reference checkpoint, define token degradation
$d_i(x;\bits)=-\log p_{\bits}(x_i\mid x_{<i})+\log p_{\mathrm{full}}(x_i\mid x_{<i})$.
The audit-only FWP increment is
\begin{equation}
\Delta \ell_{\FWP}^{\mathrm{audit}}(x;\bits)
=\frac{1}{|I_p(x)|}\sum_{i\in I_p(x)}w_i(x)d_i(x;\bits),
\label{eq:fwp-delta}
\end{equation}
while $\Delta\overline\ell(x;\bits)=|I_p(x)|^{-1}\sum_i d_i(x;\bits)$ is ordinary prompt-average $\Delta$NLL.

\begin{proposition}[Within-request FWP amplification]
\label{prop:model-fwp-amplification-main}
For every request with finite token degradation and the normalized weights in Eq.~\eqref{eq:fwp-weights},
\begin{equation}
\Delta \ell_{\FWP}^{\mathrm{audit}}(x;\bits)-\Delta\overline\ell(x;\bits)
=\operatorname{Cov}_{i\in I_p(x)}\!\left(w_i(x),d_i(x;\bits)\right).
\label{eq:fwp-covariance-main}
\end{equation}
Consequently, FWP amplifies average degradation exactly when inverse-router-norm weights have positive within-request covariance with token degradation.
\end{proposition}

The identity follows from $\mathbb E_i[w_i]=1$ and is proved in Appendix~\ref{app:qb-proof}. It matches the deployed within-request normalization, unlike a population-reweighted construction. In the two-expert top-$1$ model, if every token in a request uses the same expert, $w_i$ is constant and the covariance is zero; that model motivates fragility but does not prove deployed FWP amplification. The empirical prompt-only signal tests whether the covariance mechanism remains informative in realistic top-$k$ requests.

\subsection{LP-Based Adaptive Routing}
\label{sec:theory-lp}

We now use the quality estimates from Sections~\ref{sec:theory-sensitivity}--\ref{sec:theory-fwp} to derive a routing policy. We formulate adaptive routing as a window-level linear program and obtain the per-request score from its dual.

\paragraph{Window-level LP.}
The serving pool $\instances$ is fixed and heterogeneous. Requests may be grouped into classes $g \in \classes$ by prompt length, task type, or affinity pattern. Let $L_g$ be the expected output-token count, $s_{ga}(t)$ the profiled service cost (from TTFT/TPOT), $B_g(t)$ the class demand in window $t$, $C_a(t)$ the capacity of instance $a$, and $\tau_g$ the per-class quality tolerance. The LP selects class-instance flows $x_{ga}(t) \ge 0$ to maximize served tokens subject to demand, capacity, and quality-risk constraints:
\begin{equation}
\max_{x \ge 0}\; \sum_{g,a} L_g x_{ga}(t)
\;\; \text{s.t.}\;\;
\sum_a x_{ga} \le B_g,\;
\sum_g s_{ga} x_{ga} \le C_a,\;
\sum_a (\losshat_{ga} - \tau_g) x_{ga} \le 0,
\label{eq:window-lp}
\end{equation}
where $\losshat_{ga}(t)$ is a calibrated estimate of mean $\Delta$NLL for class $g$ on instance $a$. The signed coefficient $\losshat_{ga}-\tau_g$ allows below-budget assignments to offset above-budget assignments because the constraint controls the class average. Replacing it by a positive-part hinge would define a different optimization problem and is not used in the KKT claim.

\paragraph{Dual structure and per-request score.}
Let $\pi_a \ge 0$ be the price of capacity on instance $a$, and $\nu_g \ge 0$ the price of class-average quality risk. Define $q_{ga}(t)=\losshat_{ga}(t)-\tau_g$. The Lagrangian of Eq.~\eqref{eq:window-lp} yields the one-step reduced reward
\begin{equation}
R_{ga}(t) = L_g - \pi_a(t)s_{ga}(t) - \nu_g(t)q_{ga}(t),
\end{equation}
up to a class-specific demand price. A practical quota-deficit term $d_{ga}(t)$ may be added to track a precomputed target allocation,
\begin{equation}
S_{ga}(t) = \beta d_{ga}(t) + L_g - \pi_a(t)s_{ga}(t) - \nu_g(t)q_{ga}(t),
\label{eq:final-score}
\end{equation}
where $\beta \ge 0$ controls tracking. The exact LP reduced reward is the $\beta=0$ case; the tracking term is an implementation regularizer rather than a new LP constraint.

\begin{theorem}[KKT consistency with the LP optimum]
\label{thm:lp-consistency-main}
Let $x^*$ be a primal optimum and $(\lambda^*,\pi^*,\nu^*)$ a dual optimum of Eq.~\eqref{eq:window-lp}. For $\beta=0$, every pair with $x^*_{ga}>0$ maximizes $R_{ga}$ over $a$ at prices $(\pi^*,\nu^*)$. If ties are allocated according to the positive-flow proportions of $x^*$, the resulting allocation is KKT-consistent with the LP optimum. An arbitrary argmax at the same prices satisfies the stationarity condition but need not satisfy primal capacity or quality feasibility.
\end{theorem}

Theorem~\ref{thm:lp-consistency-main} identifies the exact guarantee and its boundary. The score inherits throughput, capacity, and signed quality terms from the LP, but price convergence and primal-feasible tie handling are separate requirements. For a fixed price snapshot the data-plane decision is
\begin{equation}
a_t^*(g) = \arg\max_{a \in \instances} S_{ga}(t),
\label{eq:online-rule}
\end{equation}
which requires $O(|\instances|)$ score evaluations per request, plus the reference prefill and a target prefill whenever KV state cannot be reused. Prices update once per control window. The current experiments evaluate offline population allocation and component timings; they do not claim that the complete online dual-price controller has been deployed (Algorithm~\ref{alg:lightweight-routing}, Appendix~\ref{app:eval-protocol}).

\section{Results and Discussion}
\label{sec:results}

The evaluation isolates fixed-pool routing value without fusing incompatible protocols. We therefore present a strict ladder: complete-instance quality, fixed-population allocation, prompt-only signal validity, and resource accounting. Numbers from different quantizers, checkpoints, capacity models, or live topologies are not fused.

\paragraph{Setup.}
The complete-instance backbone uses Qwen3-30B-A3B with every one of its 6,144 expert blocks quantized by HQQ in each W2, W3, and W4 instance. The extended population contains 88 prompts from short QA, code, and long-context tasks. DeepSeek-V2-Lite results in the earlier study change sampled experts in four layers; we retain them only as architecture diagnostics and make no complete-instance cross-model claim. Appendix~\ref{app:eval-protocol} records the protocol and evidence boundaries.

\paragraph{Evaluation metric.}
For request $t$ routed to instance $a$,
\begin{equation}
U_t(a) = \alpha_s L_{g_t} - \alpha_q\,\widehat{\Delta \mathcal{Q}}_t(a) - \beta\,\mathrm{Latency}_t(a) - \lambda_{\mathrm{SLO}}\,\mathbf{1}\!\left\{\widehat{\Delta \mathcal{Q}}_t(a) > \tau_{g_t}\right\},
\label{eq:wcu-utility}
\end{equation}
and a policy $\pi$ is scored by its cumulative gap to the per-step quality-optimal comparator,
\begin{equation}
\Regret_T(\pi) = \sum_{t=1}^{T} \Bigl(U_t(a_t^{\star}) - U_t(a_t^{\pi})\Bigr),
\label{eq:wcu-regret}
\end{equation}
with $a_t^{\star}$ chosen per request given perfect quality information.\footnote{The per-step quality-optimal comparator is a regret normalization, not a globally optimal scheduler; it ignores queue state and can be beaten on aggregate utility by load-aware policies that trade small, controlled quality losses for queue stability. This appears symmetrically on Qwen and DeepSeek and is consistent with the LP objective being request-local.}

\subsection{Complete-Instance Quality and Request Heterogeneity}
\label{sec:results-fragility}

\begin{table}[t]
\centering
\small
\caption{Complete-instance Qwen quality degradation on the same 88 extended prompts. Every row quantizes all 6,144 expert blocks; values are request-level average-NLL increases relative to the reference checkpoint.}
\label{tab:complete-instance-quality}
\begin{tabular}{@{}lrrr@{}}
\toprule
& \textbf{W2} & \textbf{W3} & \textbf{W4} \\
\midrule
Mean   & $0.9437$ & $0.1832$ & $0.0513$ \\
Median & $0.8514$ & $0.1387$ & $0.0241$ \\
90th percentile & $1.4517$ & $0.4081$ & $0.1800$ \\
\bottomrule
\end{tabular}
\end{table}

Table~\ref{tab:complete-instance-quality} reports complete-instance distributions rather than sampled-block profiles. W3 is infeasible at the headline mean budget introduced below, but its median is below that budget, establishing request heterogeneity without inferring causality from sampled-expert probes. In a separate matched-3.0-average-bit experiment, all-expert W3 has mean $\Delta$NLL $0.109$, while tested router-norm-guided, activation-frequency-guided, and random 2/4-bit allocations yield $0.315$, $0.318$, and $0.334$. This separate protocol does not establish that all-expert W3 dominates every possible non-uniform allocation; it only identifies the strongest tested arm.

\subsection{Fixed-Population Allocation and Gain Attribution}
\label{sec:results-routing}

The cleanest comparison fixes the complete W2/W3/W4 pool, all 88 requests, one quality budget $\tau=\operatorname{mean}(\Delta\mathrm{NLL}_{\mathrm{W4}})+0.10=0.1513$, and analytical per-request multipliers $2$, $4/3$, and $1$. The multiplier is a model-based population allocation quantity; it is not measured routed-pool throughput.

\begin{table*}[t]
\centering
\small
\setlength{\tabcolsep}{5pt}
\caption{Offline fixed-population allocation with complete Qwen W2/W3/W4 losses. Parenthesized multipliers violate the common mean-$\Delta$NLL budget. The true-loss row orders requests using measured W3 loss and is a restricted analysis reference, not a deployable global oracle.}
\label{tab:routing-fractions}
\begin{tabular}{@{}lrrrrr@{}}
\toprule
\textbf{Policy} & \textbf{W2\%} & \textbf{W3\%} & \textbf{W4\%} & \textbf{Mean $\Delta$NLL} & \textbf{Multiplier} \\
\midrule
Static W4 & $0$ & $0$ & $100$ & $0.051$ & $1.000\times$ \\
Static W3 & $0$ & $100$ & $0$ & $0.183$ & $(1.333\times)$ \\
Static W2 & $100$ & $0$ & $0$ & $0.944$ & $(2.000\times)$ \\
Capacity-proportional mixture & $46$ & $31$ & $23$ & $0.504$ & $(1.564\times)$ \\
Request-agnostic budget-aware mixture & $0$ & $76$ & $24$ & $0.151$ & $1.253\times$ \\
Sequence length & $0$ & $56$ & $44$ & $0.147$ & $1.186\times$ \\
Activation frequency & $0$ & $59$ & $41$ & $0.150$ & $1.197\times$ \\
\textbf{FWP} & $\mathbf{0}$ & $\mathbf{85}$ & $\mathbf{15}$ & $\mathbf{0.151}$ & $\mathbf{1.284\times}$ \\
Restricted true-loss ordering & $0$ & $91$ & $9$ & $0.149$ & $1.303\times$ \\
\bottomrule
\end{tabular}
\end{table*}

Table~\ref{tab:routing-fractions} isolates the source of the headline multiplier. Request-agnostic budget allocation contributes $0.253$ of the $0.284$ multiplier increase over static W4. FWP adds $0.031$ multiplier points, or $2.5\%$ relative to the request-agnostic mixture, and closes approximately $62\%$ of the remaining gap to the restricted true-loss ordering. Thus, the $1.284\times$ multiplier should not be interpreted as a $28.4\%$ FWP throughput gain: within this fixed complete-instance population, allocation yields most of the modeled gain and FWP supplies a smaller request-specific increment.

\subsection{Prompt-Only FWP Validity and Abstention}
\label{sec:results-fwp}

Same-instance FWP correlations of $0.999/0.998$ on complete W2 Qwen are measurement sanity checks because the weighted NLL and target loss share the same forward pass. The deployable question is whether a prompt-time statistic on one reference instance predicts risk on another instance. In that harder setting, the prompt-only W2 score predicts W3 risk with Spearman $0.586$ and AUROC $0.747$. We therefore describe FWP as a calibrated feature, not a direct quality measurement.

\begin{table}[t]
\centering
\small
\caption{Correlation between prefix FWP and full-prompt quality loss on 64 extended prompts.}
\label{tab:prefix-boundary}
\begin{tabular}{@{}rr@{}}
\toprule
\textbf{Prefix tokens} & \textbf{Spearman} \\
\midrule
$64$ & $0.238$ \\
$128$ & $0.505$ \\
$256$ & $0.593$ \\
$512$ & $0.765$ \\
$1024$ & $0.999$ \\
\bottomrule
\end{tabular}
\end{table}

Table~\ref{tab:prefix-boundary} makes prompt length a material operating condition. A production policy must calibrate by length and route requests below a validated prefix threshold to a conservative higher-bitwidth fallback. Moreover, the small held-out 50/50 calibration splits achieve only approximately $35$--$50\%$ budget-compliant splits. Population frontiers satisfy the budget by construction, but they do not establish held-out risk control. We consequently make no deployment-level quality guarantee; such a claim requires disjoint task-family splits, confidence bounds, and a predeclared compliance target.

\subsection{Memory, Endpoint Capacity, and the Replica Boundary}

\begin{table}[t]
\centering
\small
\caption{Qwen parameter-memory estimates on one accounting basis. These numbers exclude KV cache, runtime buffers, allocator fragmentation, and any separate scorer.}
\label{tab:memory-accounting}
\begin{tabular}{@{}lr@{}}
\toprule
\textbf{Resident set} & \textbf{GiB} \\
\midrule
One W4 instance & $18.06$ \\
W2/W3/W4 pool & $44.05$ \\
Three W4 replicas & $54.17$ \\
One BF16 instance & $56.87$ \\
\bottomrule
\end{tabular}
\end{table}

The fixed heterogeneous pool costs $25.99$ GiB more parameter memory than one W4 instance but less than three W4 replicas or one BF16 instance under this accounting. This does not make it the optimal fleet: KV capacity, batching, tensor parallelism, and endpoint kernels can reverse a parameter-only comparison. Equation~\eqref{eq:outer-provisioning-boundary} therefore leaves replica counts to the upstream provisioner, and the routing layer can operate over heterogeneous instances, identical replicas, or both.

In a separate same-two-A100 endpoint experiment with prefix caching disabled, 216 requests are repeated five times; 128 prompts contain 1,023 input tokens followed by 128 decoded tokens. W2 with two data-parallel replicas reaches $3270\pm33$ output tokens/s, W3 with two-way tensor parallelism reaches $2591\pm11$, and W4 with two-way tensor parallelism reaches $2334\pm17$. These measured endpoint rates establish a capacity ordering for that backend and protocol. They are not the $1.284\times$ allocation multiplier and are not combined with HQQ quality measurements as though they came from identical checkpoints.

Component overheads are likewise reported separately: FWP instrumentation adds $3.1\%$ and $4.9\%$ to one 512- and 1,024-token reference prefill; a four-class CPU LP takes $1.47/1.52/1.57$ ms at p50/p95/p99 once per window; and a three-instance score comparison takes $1.63\,\mu$s per request. These measurements do not include reference-queue delay, target re-prefill, RPC dispatch, or end-to-end concurrent serving.

Finally, the live routed prototype is not an equal-total-resource replica comparison: it uses three serving GPUs plus a separate INT3 scorer GPU, and an optimized same-resource replica fleet was not measured. We therefore omit its ratios from the headline and do not claim that adaptive routing dominates adding replicas. The paper's validated conclusion is conditional: given a fixed resident pool, request-specific risk can improve allocation over request-agnostic mixing in the offline complete-instance population.

\section{Conclusion}

This paper solves a conditional control problem: route complete requests over a fixed, pre-materialized pool of quantized MoE instances. Replica counts, instance composition, memory residence, and reconfiguration occur at a slower provisioning timescale. A provisioner may choose additional identical replicas, a heterogeneous pool, or a mixture of both; the routing layer consumes that choice and does not claim global fleet optimality.

The corrected complete-Qwen study shows why the inner problem remains meaningful. At $\tau=0.1513$, static W3 is infeasible in mean, request-agnostic mixing reaches a $1.253\times$ offline model-based multiplier, and FWP reaches $1.284\times$. The incremental request-specific value is therefore $2.5\%$ relative, not the full $28.4\%$ difference from static W4. Prompt-only prediction, short-prefix behavior, held-out compliance, second-prefill cost, and the missing equal-resource replica experiment remain explicit limitations. The evidence supports fixed-pool allocation research; it does not yet support an end-to-end deployment guarantee.

\bibliographystyle{abbrvnat}
\bibliography{references}

\appendix

\section{Extended Related Work and Positioning}
\label{sec:related-extended}

This appendix expands the compact discussion in Section~\ref{sec:related} along the four lines of work that frame our contribution.

\paragraph{Sparse MoE architectures and routing.}
The conditional-computation idea behind sparsely activated experts \citep{shazeer2017outrageously} was scaled and stabilized through GShard, Switch Transformers, GLaM, and ST-MoE \citep{lepikhin2021gshard,fedus2022switch,du2022glam,zoph2022stmoe}, with router-design alternatives such as BASE Layers, Hash Layers, and Expert-Choice routing addressing load balance and stability \citep{lewis2021base,roller2021hash,zhou2022expertchoice}. Open competitive MoE checkpoints---Mixtral, DeepSeekMoE, DeepSeek-V2/V3, and Qwen MoE---now place expert specialization at the center of LLM design \citep{jiang2024mixtral,dai2024deepseekmoe,deepseekai2024deepseekv2,deepseekv3,qwen2moe2024,qwen3technical2025}, and Mixture-of-Depths extends the same routing logic to token-level compute allocation \citep{raposo2024mod}. We treat this body of work as the architectural substrate. Our contribution is not a new MoE architecture or training recipe; we ask how heterogeneously quantized instances of a given MoE checkpoint should be evaluated and routed at serving time, once expert affinity, expert fragility, and request heterogeneity are measured.

\paragraph{Post-training quantization and expert-wise bit allocation for LLMs.}
Low-bit LLM compression has matured rapidly through methods including LLM.int8(), ZeroQuant, GPTQ, SmoothQuant, AWQ, SparseGPT, OmniQuant, QLoRA, SpinQuant, and HQQ \citep{dettmers2022llmint8,yao2022zeroquant,frantar2023gptq,xiao2023smoothquant,lin2024awq,frantar2023sparsegpt,shao2024omniquant,dettmers2023qlora,liu2024spinquant,badri2023hqq}. Non-uniform bit allocation has long been informed by Hessian- and sensitivity-based methods such as HAWQ, HAQ, and BRECQ \citep{dong2019hawq,wang2019haq,li2021brecq}, which provide the conceptual ancestry for the per-expert fragility surrogate we evaluate. MoE-specific quantization is now its own subline---QMoE compresses trillion-parameter sparse models to sub-bit precision, MoQE and MC-MoE study expert-wise bit assignment, and \citet{chowdhury2026efficient} prove that an expert's change in router-vector $\ell_2$ norm during training, $\Lambda_s$, is a theoretically grounded indicator of which experts must be kept at higher bitwidth \citep{frantar2024qmoe,kim2023moqe,huang2024mcmoe,chowdhury2026efficient}. The FWP proxy in Section~\ref{sec:theory} inherits this $\Lambda$-as-fragility-signal idea and lifts it from an offline per-expert bit-allocation prior to an online per-request quality score. Our paper is complementary: rather than proposing another quantizer, we assume that a small pool of quantized instances has already been materialized by these methods and study the downstream decision problem---which instance should serve which request, and which quality signal is reliable enough to drive that decision. A checkpoint-level quality average can be an excellent audit number while still being too coarse for workload-conditioned routing, which is the gap our FWP signal targets.

\paragraph{LLM serving systems and request scheduling.}
Production-grade serving systems have made LLM inference dramatically more efficient through iteration-level scheduling, paged attention, offloading, stall-free batching, prefill/decode disaggregation, autoscaling, and program-level runtime optimization---Orca, vLLM, FlexGen, Sarathi-Serve, DistServe, SGLang, AlpaServe, Splitwise, and Llumnix \citep{yu2022orca,kwon2023pagedattention,sheng2023flexgen,agrawal2024sarathi,zhong2024distserve,zheng2024sglang,li2023alpaserve,patel2024splitwise,sun2024llumnix}, with speculative-decoding methods further reshaping the latency/throughput frontier \citep{leviathan2023speculative,chen2023specinfer}. We rely on the systems insight that throughput and latency are workload properties, not static model properties. The missing axis in most serving evaluations is quantization-induced quality risk; our lightweight routing experiments add a request-level quality term to the usual latency and queueing terms while remaining careful not to claim a full production scheduler.

\paragraph{Adaptive inference and request-level model routing.}
A growing line of work routes requests across heterogeneous models to trade quality for cost: FrugalGPT cascades models by confidence \citep{chen2023frugalgpt}, RouteLLM and RouterBench learn binary or multi-way model selectors \citep{ong2024routellm,hu2024routerbench}, Hybrid-LLM and AutoMix mix small and large models on a per-request basis \citep{ding2024hybridllm,madaan2023automix}, and CALM uses early-exit confidence inside a single model \citep{schuster2022calm}. Closest in spirit, these routers select \emph{which model} to call; we route across \emph{which bitwidth instance of the same MoE checkpoint}, and the decision lever is the expert-fragility lens rather than aggregate accuracy gaps between distinct models. Their confidence-based scoring also assumes that the candidate models are themselves well-calibrated outputs, whereas in our setting the very object being predicted---per-request quality degradation under a quantized instance---is what the FWP proxy is designed to measure.

\paragraph{Calibration, selective prediction, and risk-controlled inference.}
Our use of FWP proxies sits within a longer line on calibration and selective prediction, which insists that a model's confidence or risk score be evaluated against downstream decisions rather than only against aggregate accuracy \citep{guo2017calibration,geifman2017selective}. Distribution-free risk control via conformal prediction and Risk-Controlling Prediction Sets makes those guarantees finite-sample \citep{angelopoulos2021gentle,bates2021rcps}, and recent work extends conformal control to language-model outputs and generation-time abstention \citep{quach2024conformallm,mohri2024conformal}. In our setting the calibrated object is request-level quality degradation under a quantized MoE instance, not class probability; the small audited quality-risk budget in our lightweight routing policy is the operational analog of a target risk level. Static average NLL and perplexity remain necessary audit targets, but the decisive question is whether a proxy helps a serving policy reduce workload-conditioned utility gap under realistic information and cost assumptions.

\section{A Self-Contained Derivation of the Reduced-Form Quality Surrogate Under Explicit Model Assumptions}
\label{app:qb-proof}

This appendix replaces the generic perturbation argument with a fully specified two-expert sparse-MoE model. The development is staged in the standard order: first the data-generating process, router, expert margins, and quantization noise; then the induced quality-loss expression; and finally the connection from the resulting coefficient map to the structured proxy used in the main paper.

\subsection{Proof Architecture}

The derivation has five steps.
\begin{enumerate}[leftmargin=*, itemsep=2pt, topsep=2pt]
\item Specify a two-expert sparse MoE in which requests belong to a rare population of mass $\alpha$ and a common population of mass $1-\alpha$, and each population is routed to its own expert.
\item Define quality as negative expected logistic loss and show that the quality loss under quantization decomposes exactly into an affinity term and an expert-local degradation term.
\item Apply a small-noise expansion to each local degradation term to obtain a reduced-form surrogate of the form ``coefficient $\times$ bit-decay'', then connect the coefficient to a router-derived margin proxy.
\item Instantiate the same model with FWP-style weights and prove that the weighted proxy amplifies the degradation of the rare, fragile expert.
\item Convert the same weighted surrogate into a request-level utility comparison, which makes the connection to the routing objective in Eq.~\eqref{eq:wcu-utility} explicit.
\end{enumerate}

The gain of this route is that the additive surrogate is no longer postulated first and justified later. It is derived from an explicit mathematical model.

The appendix is local rather than global. It justifies the form of the empirical surrogate and the FWP-style amplification mechanism under transparent assumptions; it does not prove that real MoE checkpoints obey an exact additive law globally. The main text accordingly treats the fitted surrogate as a weak ranking prior and uses measured FWP signals for the strongest routing claims.

\subsection{Two-Expert Sparse-MoE Model}

We consider a sparse top-1 MoE classifier with two experts, denoted by $e_r$ and $e_c$. Each request belongs to one of two populations:
\begin{equation}
X \in \{r,c\},
\qquad
\Pr(X=r)=\alpha,
\qquad
\Pr(X=c)=1-\alpha,
\qquad
0<\alpha<\frac{1}{4}.
\end{equation}
Population $r$ is the rare population and population $c$ is the common population. The router is deterministic in this model:
\begin{equation}
g(x)=
\begin{cases}
e_r, & X=r, \\
e_c, & X=c.
\end{cases}
\label{eq:model-router}
\end{equation}
This is the simplest mathematical extraction of sparse MoE behavior: each request only activates a small subset of experts, and in the present specification that subset has size one.

For a request of population $j \in \{r,c\}$, let $m_j>0$ be the full-precision signed classification margin of the routed expert. We assume
\begin{equation}
0 < m_r < m_c,
\label{eq:model-margin-order}
\end{equation}
which captures the fragility mechanism of \citet{chowdhury2026efficient} in its simplest form: the rare expert has a weaker activation margin and is therefore more fragile under quantization perturbations.

The prediction loss is logistic:
\begin{equation}
\ell(z) := \log\!\left(1+e^{-z}\right).
\label{eq:model-logistic}
\end{equation}
The full-precision quality is defined by
\begin{equation}
Q_{\mathrm{full}} := -\mathbb{E}\bigl[\ell(m_X)\bigr].
\end{equation}

\subsection{Quantization Model}

If expert $e_j$ is stored at bitwidth $b_j$, quantization perturbs the signed margin by a scalar random variable $\xi_j(b_j)$, so that the quantized margin for a request of type $j$ becomes
\begin{equation}
M_j(b_j) := m_j + \xi_j(b_j).
\label{eq:model-quantized-margin}
\end{equation}
We assume the following noise model for each $j \in \{r,c\}$:
\begin{align}
\mathbb{E}[\xi_j(b_j)] &= 0, \\
\mathbb{E}[\xi_j(b_j)^2] &= v_j 2^{-2b_j}, \\
\mathbb{E}[|\xi_j(b_j)|^3] &\le \chi_j 2^{-3b_j}, \\
|\xi_j(b_j)| &\le \bar\delta_j 2^{-b_j} \quad \text{a.s.}
\label{eq:model-noise-moments}
\end{align}
for constants $v_j, \chi_j, \bar\delta_j > 0$. The variance scaling in Eq.~\eqref{eq:model-noise-moments} is the standard uniform-quantization scaling under a mean-square criterion.

Define the quantized quality by
\begin{equation}
Q(\bits)
:=
-\mathbb{E}\bigl[\ell(M_X(b_X))\bigr],
\qquad
\bits=(b_r,b_c).
\label{eq:model-quality}
\end{equation}
The induced quality loss is
\begin{equation}
\loss(\bits) := Q_{\mathrm{full}} - Q(\bits).
\label{eq:model-loss-def}
\end{equation}

\subsection{Exact Affinity--Fragility Decomposition}

The first key fact is exact rather than asymptotic.

\begin{proposition}[Exact mixture decomposition]
\label{prop:model-mixture}
Under the model above,
\begin{equation}
\loss(\bits)
=
\alpha\,\Delta_r(b_r)
+
(1-\alpha)\,\Delta_c(b_c),
\label{eq:model-mixture}
\end{equation}
where
\begin{equation}
\Delta_j(b_j)
:=
\mathbb{E}\bigl[\ell(M_j(b_j)) - \ell(m_j)\bigr],
\qquad j \in \{r,c\}.
\label{eq:model-local-loss}
\end{equation}
\end{proposition}

\begin{proof}
Condition on the request type $X$. By definition,
\begin{equation}
\loss(\bits)
=
\mathbb{E}\Bigl[\ell(M_X(b_X)) - \ell(m_X)\Bigr].
\end{equation}
Using the law of total expectation and the probabilities of the two populations gives
\begin{equation}
\loss(\bits)
=
\Pr(X=r)\,\mathbb{E}\bigl[\ell(M_r(b_r)) - \ell(m_r)\bigr]
+
\Pr(X=c)\,\mathbb{E}\bigl[\ell(M_c(b_c)) - \ell(m_c)\bigr],
\end{equation}
which is exactly Eq.~\eqref{eq:model-mixture}.
\end{proof}

Proposition~\ref{prop:model-mixture} already factors into two objects that should not be combined. The probabilities $\alpha$ and $1-\alpha$ are affinity weights. The local terms $\Delta_r$ and $\Delta_c$ are intrinsic degradation functions of the routed experts.

\subsection{Small-Noise Expansion of the Local Degradation}

We now derive the reduced form by expanding the local term in Eq.~\eqref{eq:model-local-loss}. Let
\begin{equation}
\underline{m}_j := m_j - \bar\delta_j 2^{-b_j}.
\end{equation}
Assume $\underline{m}_j > 0$, so the quantized margin remains in a compact subset of the positive half-line. Define
\begin{equation}
L_j := \sup_{z \ge \underline{m}_j} |\ell'''(z)|.
\end{equation}

\begin{lemma}[Second-order local expansion]
\label{lem:model-local-expansion}
For each $j \in \{r,c\}$,
\begin{equation}
\Delta_j(b_j)
=
\frac{1}{2}\ell''(m_j) v_j 2^{-2b_j}
+ r_j(b_j),
\label{eq:model-local-expansion}
\end{equation}
where the remainder satisfies
\begin{equation}
|r_j(b_j)| \le \frac{L_j}{6}\chi_j 2^{-3b_j}.
\label{eq:model-local-remainder}
\end{equation}
\end{lemma}

\begin{proof}
Taylor-expand $\ell(m_j + \xi_j)$ around $m_j$:
\begin{equation}
\ell(m_j + \xi_j)
=
\ell(m_j)
+ \ell'(m_j)\xi_j
+ \frac{1}{2}\ell''(m_j)\xi_j^2
+ \frac{1}{6}\ell'''(\widetilde{m}_j)\xi_j^3,
\end{equation}
where $\widetilde{m}_j$ lies between $m_j$ and $m_j+\xi_j$. Taking expectations and using $\mathbb{E}[\xi_j]=0$ removes the first-order term. The second-order term becomes $\frac{1}{2}\ell''(m_j)\mathbb{E}[\xi_j^2] = \frac{1}{2}\ell''(m_j)v_j2^{-2b_j}$. For the remainder,
\begin{equation}
\left|\frac{1}{6}\mathbb{E}\bigl[\ell'''(\widetilde{m}_j)\xi_j^3\bigr]\right|
\le
\frac{L_j}{6}\mathbb{E}[|\xi_j|^3]
\le
\frac{L_j}{6}\chi_j 2^{-3b_j},
\end{equation}
which proves Eq.~\eqref{eq:model-local-remainder}.
\end{proof}

Lemma~\ref{lem:model-local-expansion} is the precise point where bitwidth enters the centered-noise theory. The leading term is proportional to $2^{-2b_j}$ because the first-order term vanishes and the dominant contribution is the variance. A $2^{-b_j}$ term requires the distinct biased-noise expansion below and carries the signed coefficient $\ell'(m_j)m_j^{(1)}$; it cannot reuse the positive centered-noise coefficient.

\begin{lemma}[Biased-noise first-order variant]
\label{lem:model-first-order}
Suppose instead that the quantization perturbation has a small mean-shift component
\begin{equation}
\mathbb{E}[\xi_j(b_j)] = m^{(1)}_j2^{-b_j} + m^{(2)}_j2^{-2b_j} + \mathcal{O}(2^{-3b_j}),
\end{equation}
while its second and third moment bounds follow the scaling in Eq.~\eqref{eq:model-noise-moments}. Then the local degradation admits the expansion
\begin{equation}
\Delta_j(b_j)
=
\ell'(m_j)m^{(1)}_j2^{-b_j}
+
\left(\ell'(m_j)m^{(2)}_j + \frac{1}{2}\ell''(m_j)v_j\right)2^{-2b_j}
+ \mathcal{O}(2^{-3b_j}).
\label{eq:model-first-order-expansion}
\end{equation}
\end{lemma}

\begin{proof}
Use the same Taylor expansion as in Lemma~\ref{lem:model-local-expansion}, but do not cancel the first-order term. Taking expectations gives
\begin{equation}
\Delta_j(b_j)
=
\ell'(m_j)\mathbb{E}[\xi_j(b_j)]
+
\frac{1}{2}\ell''(m_j)\mathbb{E}[\xi_j(b_j)^2]
+
\mathcal{O}(\mathbb{E}|\xi_j(b_j)|^3).
\end{equation}
Substituting the assumed mean expansion and the moment bounds yields Eq.~\eqref{eq:model-first-order-expansion}.
\end{proof}

Lemma~\ref{lem:model-first-order} explains why the experiments compare both $2^{-b}$ and $2^{-2b}$ families. If the dominant distortion behaves like symmetric zero-mean noise, the energy-response term $2^{-2b}$ is natural. If clipping, scaling, or quantizer bias introduces a first-order mean shift, a $2^{-b}$ term can dominate over the measured range. The empirical comparison between the two families is therefore part of the model check rather than an arbitrary hyperparameter sweep.

\begin{theorem}[Reduced form under explicit model assumptions]
\label{thm:model-reduced-form}
Under Proposition~\ref{prop:model-mixture} and Lemma~\ref{lem:model-local-expansion},
\begin{equation}
\loss(\bits)
=
\alpha s_r 2^{-2b_r} + (1-\alpha)s_c 2^{-2b_c} + r(\bits),
\label{eq:model-reduced-form}
\end{equation}
where $s_j := \frac{1}{2} v_j \ell''(m_j)$ for $j \in \{r,c\}$,
and
\begin{equation}
|r(\bits)| \le \frac{L_r}{6}\alpha\chi_r 2^{-3b_r} + \frac{L_c}{6}(1-\alpha)\chi_c 2^{-3b_c}.
\label{eq:model-global-remainder}
\end{equation}
\end{theorem}

\begin{proof}
Substitute the two local expansions from Eq.~\eqref{eq:model-local-expansion} into the exact decomposition in Eq.~\eqref{eq:model-mixture} and collect terms.
\end{proof}

Theorem~\ref{thm:model-reduced-form} is the core constructive result. The reduced-form surrogate is not assumed. It is obtained by combining an exact affinity decomposition with a local Taylor expansion under symmetric quantization noise.

\subsection{Intrinsic Fragility and Rare-Expert Ordering}

To factor intrinsic fragility out of the affinity-weighted coefficient, recall the intrinsic local coefficient
\begin{equation}
s_j := \frac{1}{2} v_j \ell''(m_j),
\qquad j \in \{r,c\}.
\label{eq:model-intrinsic-fragility}
\end{equation}
The leading coefficients in Eq.~\eqref{eq:model-reduced-form} are therefore $\alpha s_r$ and $(1-\alpha)s_c$, explicitly factoring the population-level affinity from the intrinsic fragility.

\begin{proposition}[Rare expert is intrinsically more fragile]
\label{prop:model-fragility-order}
If $0<m_r<m_c$ and $v_r \ge v_c$, then
\begin{equation}
s_r > s_c.
\end{equation}
\end{proposition}

\begin{proof}
For $z>0$, the logistic loss satisfies
\begin{equation}
\ell''(z) = \frac{e^z}{(1+e^z)^2},
\qquad
\frac{d}{dz}\ell''(z) = \frac{e^z(1-e^z)}{(1+e^z)^3} < 0.
\end{equation}
Hence $\ell''(m_r) > \ell''(m_c)$ whenever $0<m_r<m_c$. Multiplying by $v_r/2$ and $v_c/2$ and using $v_r \ge v_c$ yields $s_r > s_c$.
\end{proof}

Proposition~\ref{prop:model-fragility-order} formalizes the qualitative statement that weaker-margin experts are more fragile under the same quantization-noise scale. Affinity determines how often that fragility matters; it does not define fragility itself.

\subsection{Connection to Router-Based Structured Coefficients}

One bridge back to the main paper remains: how should the margin $m_j$ be represented by observable router-side quantities? The next assumption makes that step explicit rather than implicit.

\begin{assumption}[Router--margin link]
\label{ass:model-router-margin}
For each population-specific expert $j \in \{r,c\}$, there exist constants $\kappa, \mu, \tau > 0$ and a scale correction variable $V_j \ge 0$ such that
\begin{equation}
m_j
\ge
\kappa \frac{\widetilde{\Lambda}_j + \tau}{1 + \mu V_j},
\label{eq:model-router-margin}
\end{equation}
where $\widetilde{\Lambda}_j$ is a router-derived margin proxy. This abstracts the claim in \citet{chowdhury2026efficient} that larger router-norm change corresponds to a larger activation margin.
\end{assumption}

\begin{proposition}[Monotone upper bound for structured coefficients]
\label{cor:model-structured}
Under Assumption~\ref{ass:model-router-margin},
\begin{equation}
s_j
\le
\frac{v_j}{2}
\ell''\!\left(
\kappa \frac{\widetilde{\Lambda}_j + \tau}{1 + \mu V_j}
\right),
\qquad j \in \{r,c\}.
\label{eq:model-structured-map}
\end{equation}
Therefore $s_j$ is upper-bounded by a monotone decreasing function of $(\widetilde{\Lambda}_j + \tau)/(1+\mu V_j)$.
\end{proposition}

\begin{proof}
Equation~\eqref{eq:model-intrinsic-fragility} gives $s_j = \frac{1}{2}v_j\ell''(m_j)$. Since $\ell''(z)$ is decreasing on $z>0$, substituting the lower bound on $m_j$ from Eq.~\eqref{eq:model-router-margin} yields Eq.~\eqref{eq:model-structured-map}. The right-hand side is monotone decreasing in $(\widetilde{\Lambda}_j + \tau)/(1+\mu V_j)$, giving the stated upper-bound interpretation.
\end{proof}

Proposition~\ref{cor:model-structured} is the right interpretation of the structured coefficient family used as a restricted ablation in the main text (Section~\ref{sec:theory-sensitivity}). The empirical structured family is
\begin{equation}
\widehat{c}^{\mathrm{str}}_e = \theta_{\ell(e)}\!\left(\frac{1 + \mu V_e}{\widetilde{\Lambda}_e + \tau}\right)^p,
\label{eq:structured-coef}
\end{equation}
where $\widetilde{\Lambda}_e$ is a router-derived fragility proxy (final router norm), $V_e$ is a normalized-MaxVar scale correction, and $(\mu, \tau, p)$ are shared parameters. The model does not prove that the exact coefficient must equal a specific power law of $(1+\mu V_j)/(\widetilde{\Lambda}_j + \tau)$. It proves a cleaner statement: the intrinsic fragility coefficient is controlled by a monotone decreasing upper bound in a router-derived margin proxy. The parametric family in Eq.~\eqref{eq:structured-coef} is then a parsimonious approximation to that monotone upper-bound shape when we return to real MoE checkpoints with many experts and imperfect observability.

\subsection{Conditional Multi-Layer \texorpdfstring{Top-$k$}{Top-k} Extension}
\label{sec:multilayer-extension}

The two-expert identity isolates a mechanism but does not remove the interactions present in a realistic MoE. Let \(j=(\ell,e)\) index an expert block, let \(\Theta\) collect reference expert parameters, and let \(\delta_j\) be the block perturbation produced by quantization. For a fixed request and evaluation token sequence, freeze the reference top-\(k\) support and denote the resulting smooth loss by \(\widetilde{\mathcal L}_x\).

\begin{theorem}[Conditional blockwise expansion]
\label{thm:multilayer-topk}
Assume the reference-support branch is \(C^3\) along the perturbation segment and every top-\(k\) routing margin remains positive enough that no selected and unselected expert exchange order. Then
\begin{equation}
\Delta\mathcal L_x
=
\sum_j \nabla_j\mathcal L_x^\top\delta_j
+\frac{1}{2}\sum_{j,k}\delta_j^\top H_{jk}(x)\delta_k
+R_{3,x},
\qquad
|R_{3,x}|\le \frac{M_x}{6}\|\delta\|^3.
\label{eq:multilayer-taylor}
\end{equation}
If perturbations are centered and block-uncorrelated with
\(\operatorname{Cov}(\delta_j)=2^{-2b_j}\Sigma_j\), then
\begin{equation}
\mathbb E[\Delta\mathcal L_x]
=
\frac{1}{2}\sum_j 2^{-2b_j}
\operatorname{tr}\!\left(H_{jj}(x)\Sigma_j\right)
+\mathbb E[R_{3,x}].
\label{eq:multilayer-centered}
\end{equation}
Recovering the affinity--fragility surrogate additionally requires
\begin{equation}
\frac{1}{2}\operatorname{tr}\!\left(H_{jj}(x)\Sigma_j\right)
=
\rho_j(x)c_j+\varepsilon_j(x),
\label{eq:multilayer-separability}
\end{equation}
which yields
\begin{equation}
\mathbb E[\Delta\mathcal L_x]
=
\sum_j\rho_j(x)c_j2^{-2b_j}
+\sum_j\varepsilon_j(x)2^{-2b_j}
+\mathbb E[R_{3,x}].
\label{eq:multilayer-surrogate}
\end{equation}
\end{theorem}

\begin{proof}
Positive routing margins make the actual computation equal the reference-support branch along the perturbation segment. Multivariate Taylor expansion gives Eq.~\eqref{eq:multilayer-taylor}. Taking expectations removes centered linear terms and off-diagonal covariance terms, giving Eq.~\eqref{eq:multilayer-centered}; substituting Eq.~\eqref{eq:multilayer-separability} gives Eq.~\eqref{eq:multilayer-surrogate}.
\end{proof}

The theorem is conditional rather than architecture-wide. With biased or correlated deterministic quantization, the deviation from the surrogate additionally contains
\begin{equation}
E_{\mathrm{bias}}+E_{\mathrm{cross}}+E_{\mathrm{sep}}+E_{\mathrm{Taylor}}+E_{\mathrm{flip}},
\label{eq:multilayer-corrections}
\end{equation}
where \(E_{\mathrm{flip}}\) is a distinct route-support correction, not an ordinary Taylor remainder. A reference-precision router does not eliminate it because upstream expert perturbations change later hidden states and router logits. At W2, measured Qwen retention is \(95.7\%\) for top-1 and \(82.6\%\) for the full top-8 set, while a targeted two-block experiment has an \(8.5\%\) median relative additivity residual. These diagnostics support the mechanism operationally but do not validate small-noise or fixed-support assumptions at two bits.

\subsection{Within-Request FWP Weighting}

The deployed FWP normalization is within each request, so its amplification property should also be stated within a request. Let \(\mathbb E_x[\cdot]\) denote the average over non-padding prompt tokens, let \(v_i=(\bar\Lambda_i+\epsilon_w)^{-1}\), let \(w_i=v_i/\mathbb E_x[v]\), and let \(d_i\) be the matched offline token degradation.

\begin{proposition}[Within-request covariance identity]
\label{prop:model-fwp-amplification}
For every request with finite \(v_i\) and \(d_i\),
\begin{equation}
\mathbb E_x[w_i d_i]-\mathbb E_x[d_i]
=
\frac{\operatorname{Cov}_x(v_i,d_i)}{\mathbb E_x[v_i]}
=
\operatorname{Cov}_x(w_i,d_i).
\label{eq:model-fwp-dominance}
\end{equation}
Hence FWP amplifies prompt-average degradation if and only if inverse-router-norm weights and token degradation have positive within-request covariance.
\end{proposition}

\begin{proof}
By definition, \(\mathbb E_x[w]=1\). Therefore
\[
\mathbb E_x[wd]-\mathbb E_x[d]
=
\mathbb E_x[wd]-\mathbb E_x[w]\mathbb E_x[d]
=
\operatorname{Cov}_x(w,d).
\]
Substituting \(w=v/\mathbb E_x[v]\) gives the first equality.
\end{proof}

This identity also exposes the boundary of the two-expert top-1 model. If every token in one request activates the same expert, \(v_i\) and \(w_i\) are constant, so FWP equals ordinary average NLL for that request. Population-level reweighting of rare and common requests is a different statistic and cannot prove amplification of deployed FWP. The realistic top-\(k\) experiments therefore test the covariance condition directly rather than treating it as a consequence of the two-expert mixture identity.
\section{Derivation of a Window-Level LP Routing Template}
\label{app:lp-proof}

This appendix derives the fixed-pool LP and its signed reduced-reward score. It separates the exact primal--dual statements from two implementation devices: projected price updates and quota tracking. The experiments instantiate offline population allocation and component timings; they do not claim to deploy the complete online dual-price controller.

\subsection{Primal and Dual Linear Programs}

Consider the LP in Eq.~\eqref{eq:window-lp}. For convenience, write the quality-constraint coefficient as
\begin{equation}
q_{ga}(t) := \losshat_{ga}(t) - \tau_g.
\end{equation}

\begin{proposition}[Dual LP and strong duality]
\label{prop:lp-dual}
For a fixed control window $t$, the dual of the primal LP in Eq.~\eqref{eq:window-lp} is
\begin{equation}
\begin{aligned}
\min_{\lambda,\pi,\nu \ge 0}
\quad &
\sum_{g \in \classes} \lambda_g B_g(t) + \sum_{a \in \instances} \pi_a C_a(t) \\
\text{s.t.}
\quad &
\lambda_g + \pi_a s_{ga}(t) + \nu_g q_{ga}(t) \ge L_g,
\qquad \forall g \in \classes,\ a \in \instances.
\end{aligned}
\label{eq:dual-program}
\end{equation}
Whenever Eq.~\eqref{eq:window-lp} is feasible and bounded, strong duality holds.
\end{proposition}

\begin{proof}
Write Eq.~\eqref{eq:window-lp} in the standard form
\begin{equation}
\max_{x \ge 0}\; c^\top x
\qquad
\text{s.t.}\quad
Ax \le b,
\end{equation}
where $x$ stacks the variables $x_{ga}(t)$, the vector $c$ collects the throughput coefficients $L_g$, and $A$ collects the demand, capacity, and quality-constraint coefficients. The stated dual is then the standard LP dual with nonnegative multipliers for each inequality block. Strong duality follows from the fundamental theorem of linear programming.
\end{proof}

Proposition~\ref{prop:lp-dual} is the OR foundation of the online rule. The prices $\lambda_g$, $\pi_a$, and $\nu_g$ are not free tuning terms; they are the exact dual multipliers of demand, capacity, and quality constraints.

\begin{proposition}[Shadow-price characterization of active routes]
\label{prop:lp-shadow}
Let $x^*$ solve the primal LP in Eq.~\eqref{eq:window-lp}, and let $(\lambda^*, \pi^*, \nu^*)$ solve the dual LP in Eq.~\eqref{eq:dual-program}. Then for every class-instance pair $(g,a)$,
\begin{equation}
x^*_{ga}(t) > 0
\quad \Longrightarrow \quad
L_g - \pi_a^* s_{ga}(t) - \nu_g^* q_{ga}(t) = \lambda_g^*.
\label{eq:shadow-equality}
\end{equation}
If $x^*_{ga}(t)=0$, then
\begin{equation}
L_g - \pi_a^* s_{ga}(t) - \nu_g^* q_{ga}(t) \le \lambda_g^*.
\label{eq:shadow-inequality}
\end{equation}
Consequently, a class may be split only across instances that maximize the reduced reward at the optimal shadow prices.
\end{proposition}

\begin{proof}
The complementary-slackness condition for the nonnegativity constraint on $x_{ga}(t)$ implies
\begin{equation}
x^*_{ga}(t)
\Bigl[
\lambda_g^* + \pi_a^* s_{ga}(t) + \nu_g^* q_{ga}(t) - L_g
\Bigr] = 0.
\end{equation}
If $x^*_{ga}(t)>0$, the bracketed term must vanish, which yields Eq.~\eqref{eq:shadow-equality}. If $x^*_{ga}(t)=0$, dual feasibility gives Eq.~\eqref{eq:shadow-inequality}.
\end{proof}

Proposition~\ref{prop:lp-shadow} is the class-level shadow-price characterization. It is the exact static analogue of the online rule: active routes are precisely those instances whose shadow-price-adjusted utility is maximal.

\subsection{Lagrangian Decomposition and One-Step Reduced Reward}

The corresponding Lagrangian is
\begin{equation}
\mathcal{L}(x, \lambda, \pi, \nu)
=
\sum_{g,a}
\Bigl[L_g - \lambda_g - \pi_a s_{ga}(t) - \nu_g q_{ga}(t)\Bigr]x_{ga}(t)
+ \sum_g \lambda_g B_g(t) + \sum_a \pi_a C_a(t).
\end{equation}

\begin{proposition}[Reduced reward decomposition]
\label{prop:reduced-reward}
Fix nonnegative dual variables $(\lambda, \pi, \nu)$. For any request class $g$, maximizing the Lagrangian contribution of a single incremental assignment is equivalent to choosing the instance with largest reduced reward
\begin{equation}
R_{ga}(t)
=
L_g - \pi_a(t)s_{ga}(t) - \nu_g(t)q_{ga}(t),
\end{equation}
up to the class-specific constant $\lambda_g$.
\end{proposition}

\begin{proof}[Proof of Proposition~\ref{prop:reduced-reward}]
Fix nonnegative dual variables $(\lambda, \pi, \nu)$. Consider a single incremental assignment of one unit of class-$g$ traffic. Let $\delta_a \in \{0,1\}$ indicate whether that unit is sent to instance $a$, with $\sum_a \delta_a = 1$. The Lagrangian contribution of this one-step assignment is
\begin{equation}
\sum_{a \in \instances}
\Bigl[L_g - \lambda_g - \pi_a s_{ga}(t) - \nu_g q_{ga}(t)\Bigr]\delta_a.
\end{equation}
Since the class-specific term $\lambda_g$ does not depend on $a$, maximizing over the simplex is equivalent to choosing the instance with largest coefficient
\begin{equation}
R_{ga}(t) = L_g - \pi_a(t)s_{ga}(t) - \nu_g(t)q_{ga}(t).
\end{equation}
This is exactly the reduced reward stated in Proposition~\ref{prop:reduced-reward}.
\end{proof}

Proposition~\ref{prop:reduced-reward} is the one-step version of Proposition~\ref{prop:lp-shadow}. The multiplier $\lambda_g$ acts as a class-specific reservation value and drops out when comparing instances within the same class, which is why the online score only needs the capacity and quality prices.

\paragraph{Risk-formulation identity.}
The primal constraint, dual program, reduced reward, quota-tracking score, and pseudocode all use the same signed residual \(q_{ga}=\losshat_{ga}-\tau_g\). A one-sided hinge would remove the credit from below-budget assignments and define a different primal problem; no KKT claim in this paper applies to that alternative.
\subsection{Reduced Dual Objective and Price Updates}

The dual derivation can be pushed one step further by eliminating the class-demand multipliers. For fixed $(\pi,\nu)$, the smallest nonnegative $\lambda_g$ satisfying Eq.~\eqref{eq:dual-program} is
\begin{equation}
\lambda_g(\pi,\nu)
=
\Bigl[\max_{a \in \instances}\bigl\{L_g - \pi_a s_{ga}(t) - \nu_g q_{ga}(t)\bigr\}\Bigr]_+.
\label{eq:lambda-closed-form}
\end{equation}
Substituting Eq.~\eqref{eq:lambda-closed-form} into Eq.~\eqref{eq:dual-program} gives the reduced dual objective
\begin{equation}
\widetilde{D}_t(\pi,\nu)
=
\sum_{g \in \classes} B_g(t)
\Bigl[\max_{a \in \instances}\bigl\{L_g - \pi_a s_{ga}(t) - \nu_g q_{ga}(t)\bigr\}\Bigr]_+
+
\sum_{a \in \instances} \pi_a C_a(t).
\label{eq:reduced-dual-objective}
\end{equation}
Equation~\eqref{eq:reduced-dual-objective} shows that the dual optimization problem is entirely driven by reduced rewards.

\begin{lemma}[Subgradient characterization of price residuals]
\label{lem:dual-subgradient}
Let $x^*(\pi,\nu)$ be any allocation induced by active maximizers in Eq.~\eqref{eq:reduced-dual-objective}: for each class $g$, it places mass $B_g(t)$ on instances attaining a positive maximum reduced reward, places zero mass if the positive part is inactive, and may split ties arbitrarily. Then the vectors
\begin{align}
g_a^{\pi}(\pi,\nu)
&:=
C_a(t) - \sum_{g \in \classes} s_{ga}(t)x^*_{ga}(\pi,\nu), \\
g_g^{\nu}(\pi,\nu)
&:=
- \sum_{a \in \instances} q_{ga}(t)x^*_{ga}(\pi,\nu)
\end{align}
are subgradients of the reduced dual objective $\widetilde{D}_t(\pi,\nu)$.
\end{lemma}

\begin{proof}
The reduced dual objective is a positive part of pointwise maxima of affine functions of $(\pi,\nu)$ plus the linear capacity term. By Danskin's theorem, any gradient of an active affine piece is a valid subgradient of $\widetilde{D}_t$. Holding the active allocation $x^*(\pi,\nu)$ fixed, differentiating with respect to $\pi_a$ gives $C_a(t)-\sum_g s_{ga}(t)x^*_{ga}(\pi,\nu)$, and differentiating with respect to $\nu_g$ gives $-\sum_a q_{ga}(t)x^*_{ga}(\pi,\nu)$.
\end{proof}

Lemma~\ref{lem:dual-subgradient} is the formal source of the dual price updates. A projected subgradient descent step on $\widetilde{D}_t$ has the form
\begin{align}
\pi_a^{(n+1)}
&=
\Bigl[\pi_a^{(n)} - \eta_n g_a^{\pi}(\pi^{(n)},\nu^{(n)})\Bigr]_+
=
\Bigl[\pi_a^{(n)} + \eta_n\bigl(\mathrm{UsedCapacity}_a^{(n)} - C_a\bigr)\Bigr]_+, \\
\nu_g^{(n+1)}
&=
\Bigl[\nu_g^{(n)} - \eta_n g_g^{\nu}(\pi^{(n)},\nu^{(n)})\Bigr]_+
=
\Bigl[\nu_g^{(n)} + \eta_n\sum_a q_{ga}x_{ga}^{*(n)}\Bigr]_+,
\label{eq:appendix-dual-step}
\end{align}
which is the exact optimization-theoretic analogue of the template price update. An online implementation may replace the class-level aggregate $\sum_a q_{ga}x_{ga}^{*(n)}$ by the realized average-loss residual in the window; this is simply a normalized stochastic estimate of the same subgradient component.

\begin{theorem}[Finite-time dual convergence rate]
\label{thm:dual-rate}
Fix a control window $t$ and assume that $(B_g(t), s_{ga}(t), q_{ga}(t), C_a(t))$ are stationary within that window. Let $\mathcal{K} \subseteq \mathbb{R}_+^{A+|\classes|}$ be a compact convex set containing an optimal dual solution $(\pi^*,\nu^*)$, and suppose that every subgradient of $\widetilde{D}_t$ on $\mathcal{K}$ has Euclidean norm at most $G$. Consider the projected subgradient iterations
\begin{equation}
(\pi^{(n+1)},\nu^{(n+1)})
=
\Pi_{\mathcal{K}}\Bigl((\pi^{(n)},\nu^{(n)}) - \eta_n g^{(n)}\Bigr),
\qquad
g^{(n)} \in \partial \widetilde{D}_t(\pi^{(n)},\nu^{(n)}).
\label{eq:projected-dual-descent}
\end{equation}
Then
\begin{equation}
\min_{1 \le n \le T}
\Bigl(\widetilde{D}_t(\pi^{(n)},\nu^{(n)}) - \widetilde{D}_t(\pi^*,\nu^*)\Bigr)
\le
\frac{\|(\pi^{(1)},\nu^{(1)})-(\pi^*,\nu^*)\|_2^2 + G^2\sum_{n=1}^{T}\eta_n^2}
{2\sum_{n=1}^{T}\eta_n}.
\label{eq:dual-rate-bound}
\end{equation}
In particular, if $\eta_n = \eta/\sqrt{n}$, then the bound is $\mathcal{O}((1+\log T)/\sqrt{T})$; with a horizon-tuned constant step size $\eta_n = \eta/\sqrt{T}$, the bound is $\mathcal{O}(T^{-1/2})$.
\end{theorem}

\begin{proof}
Nonexpansiveness of Euclidean projection gives
\begin{equation}
\begin{aligned}
\|(\pi^{(n+1)},\nu^{(n+1)})-(\pi^*,\nu^*)\|_2^2
&\le
\|(\pi^{(n)},\nu^{(n)})-(\pi^*,\nu^*)\|_2^2
- 2\eta_n\langle g^{(n)}, (\pi^{(n)},\nu^{(n)})-(\pi^*,\nu^*)\rangle
\\
&\quad
+ \eta_n^2 \|g^{(n)}\|_2^2.
\end{aligned}
\end{equation}
By convexity of $\widetilde{D}_t$,
\begin{equation}
\widetilde{D}_t(\pi^{(n)},\nu^{(n)}) - \widetilde{D}_t(\pi^*,\nu^*)
\le
\langle g^{(n)}, (\pi^{(n)},\nu^{(n)})-(\pi^*,\nu^*)\rangle.
\end{equation}
Substituting this bound into the previous inequality, summing from $n=1$ to $T$, and using $\|g^{(n)}\|_2 \le G$ yields
\begin{equation}
2\sum_{n=1}^{T} \eta_n
\Bigl(
\widetilde{D}_t(\pi^{(n)},\nu^{(n)})
- \widetilde{D}_t(\pi^*,\nu^*)
\Bigr)
\\
\le
\|(\pi^{(1)},\nu^{(1)})-(\pi^*,\nu^*)\|_2^2
+ G^2\sum_{n=1}^{T}\eta_n^2.
\end{equation}
Taking the minimum over $n \le T$ proves Eq.~\eqref{eq:dual-rate-bound}.
\end{proof}

Theorem~\ref{thm:dual-rate} makes the price dynamics fully technical: the online dual variables are not just interpretable as shadow prices; their dual objective gap obeys the standard projected-subgradient rate under stationary window statistics.

\subsection{Quota Tracking as an Implementation Regularizer}

Let $p_{ga}^{\star}(t)$ be the target routing share derived from the LP solution and let $\widetilde{p}_{ga}(t)$ be the current realized share inside the same control window. Consider the regularized one-step objective for a request of class $g$:
\begin{equation}
\max_{\delta \in \Delta_A}
\sum_{a \in \instances}
\Bigl[L_g - \pi_a s_{ga}(t) - \nu_g q_{ga}(t)\Bigr]\delta_a
- \frac{\beta}{2}
\sum_{a \in \instances}
\Bigl(\widetilde{p}_{ga}(t) + \delta_a - p_{ga}^{\star}(t)\Bigr)^2,
\label{eq:appendix-regularized-step}
\end{equation}
where $\Delta_A$ is the probability simplex over instances.

\begin{proposition}[Derivation of the final score]
\label{prop:appendix-final-score}
Define the quota deficit
\begin{equation}
d_{ga}(t) := p_{ga}^{\star}(t) - \widetilde{p}_{ga}(t).
\end{equation}
If the one-step assignment is integral, i.e., $\delta_a \in \{0,1\}$ and $\sum_a \delta_a = 1$, then maximizing Eq.~\eqref{eq:appendix-regularized-step} is equivalent to choosing the instance with largest score
\begin{equation}
S_{ga}(t)
=
\beta d_{ga}(t)
+ L_g
- \pi_a(t)s_{ga}(t)
- \nu_g(t)q_{ga}(t).
\end{equation}
\end{proposition}

\begin{proof}
Fix an action $a$ and set $\delta_a = 1$, $\delta_{a'} = 0$ for $a' \neq a$. Expanding the quadratic term in Eq.~\eqref{eq:appendix-regularized-step} gives
\begin{equation}
-\frac{\beta}{2}
\sum_{a' \in \instances}
\Bigl(\widetilde{p}_{ga'} - p_{ga'}^{\star}\Bigr)^2
- \beta\bigl(\widetilde{p}_{ga} - p_{ga}^{\star}\bigr)
- \frac{\beta}{2}.
\end{equation}
The first and third terms are independent of the chosen action and therefore do not affect the maximizer. Using $d_{ga} = p_{ga}^{\star} - \widetilde{p}_{ga}$, the action-dependent contribution reduces to $\beta d_{ga}$. Combining this with the reward, capacity price, and slack penalty gives the stated score.
\end{proof}

Proposition~\ref{prop:appendix-final-score} derives the tracking term from a quadratic penalty around LP target shares. The signed quality coefficient remains identical to the primal LP. For $\beta>0$ this is a regularized implementation score; only the $\beta=0$ score is the exact Lagrangian reduced reward.

\subsection{Exact KKT Scope and Tie Allocation}

The dual score characterizes which routes may carry positive primal flow, but it does not by itself choose feasible proportions among tied maximizers.

\begin{theorem}[KKT consistency with the LP optimum]
\label{thm:lp-consistency}
Let \(x^*\) solve Eq.~\eqref{eq:window-lp} and let
\((\lambda^*,\pi^*,\nu^*)\) solve Eq.~\eqref{eq:dual-program}. Then:
\begin{enumerate}[leftmargin=*, itemsep=2pt, topsep=2pt]
\item every pair with \(x^*_{ga}>0\) maximizes
\(R_{ga}=L_g-\pi_a^*s_{ga}-\nu_g^*q_{ga}\) over \(a\);
\item assigning class-\(g\) traffic among those maximizers in the proportions of \(x^*\) recovers a primal-feasible allocation and is KKT-consistent with the LP optimum;
\item an arbitrary greedy tie rule satisfies the reduced-reward stationarity condition but need not satisfy demand, capacity, or quality feasibility.
\end{enumerate}
\end{theorem}

\begin{proof}
The first statement is Proposition~\ref{prop:lp-shadow}. The second uses the primal-optimal proportions and therefore inherits primal feasibility, while strong duality and complementary slackness supply the remaining KKT conditions. The third follows because stationarity restricts support but does not determine how mass is split among tied support points.
\end{proof}

Projected subgradient convergence supplies optimal prices only under the stationarity, bounded-subgradient, step-size, and projection assumptions of Theorem~\ref{thm:dual-rate}. The quota term in Proposition~\ref{prop:appendix-final-score} is useful for tracking a target split, but KKT consistency for \(\beta>0\) additionally requires the realized proportions to converge to a primal-feasible optimum. We do not infer that property from a vanishing score perturbation alone.
\subsection{Calibration Error and Workload-Conditioned Utility Loss}

The last link is from quality-prediction error to the workload-conditioned utility in Eq.~\eqref{eq:wcu-utility}. For this step, consider the linearized one-step utility
\begin{equation}
u_{ga}(t) := L_g - \alpha_q \ell_{ga}(t) - \beta s_{ga}(t),
\label{eq:appendix-utility-true}
\end{equation}
where $\ell_{ga}(t)$ is the true average quality loss of assigning class $g$ to instance $a$ in window $t$. Let the predicted utility be
\begin{equation}
\widehat{u}_{ga}(t) := L_g - \alpha_q \losshat_{ga}(t) - \beta s_{ga}(t).
\label{eq:appendix-utility-est}
\end{equation}

\begin{proposition}[Prediction error implies one-step utility regret]
\label{prop:utility-regret}
Assume that for a fixed window,
\begin{equation}
\sup_{g \in \classes,\ a \in \instances}
|\losshat_{ga}(t) - \ell_{ga}(t)| \le \varepsilon.
\label{eq:appendix-calibration-error}
\end{equation}
Let
\begin{equation}
a_g^*(t) \in \arg\max_{a \in \instances} u_{ga}(t),
\qquad
\widehat{a}_g(t) \in \arg\max_{a \in \instances} \widehat{u}_{ga}(t).
\end{equation}
Then the one-step utility loss of acting on the predicted surrogate satisfies
\begin{equation}
0
\le
u_{g a_g^*(t)}(t) - u_{g \widehat{a}_g(t)}(t)
\le
2\alpha_q \varepsilon.
\label{eq:appendix-utility-regret}
\end{equation}
Consequently, over any horizon of $T$ requests whose classes are evaluated with the same error bound, the cumulative regret is at most $2\alpha_q \varepsilon T$.
\end{proposition}

\begin{proof}
For any action $a$,
\begin{equation}
|u_{ga}(t) - \widehat{u}_{ga}(t)|
=
\alpha_q |\losshat_{ga}(t) - \ell_{ga}(t)|
\le
\alpha_q \varepsilon.
\end{equation}
Therefore,
\begin{align}
u_{g a_g^*(t)}(t) - u_{g \widehat{a}_g(t)}(t)
={}&
\bigl[u_{g a_g^*(t)}(t) - \widehat{u}_{g a_g^*(t)}(t)\bigr] \\
&{}+
\bigl[\widehat{u}_{g a_g^*(t)}(t) - \widehat{u}_{g \widehat{a}_g(t)}(t)\bigr] \\
&{}+
\bigl[\widehat{u}_{g \widehat{a}_g(t)}(t) - u_{g \widehat{a}_g(t)}(t)\bigr].
\end{align}
The middle term is nonpositive because $\widehat{a}_g(t)$ maximizes $\widehat{u}_{ga}(t)$. The first and third terms are each bounded above by $\alpha_q \varepsilon$, which proves Eq.~\eqref{eq:appendix-utility-regret}. Summing over $T$ requests yields the cumulative statement.
\end{proof}

Proposition~\ref{prop:utility-regret} connects calibrated request-risk error to routing utility. If the selected prompt-side predictor is uniformly calibrated within error $\varepsilon$, the induced one-step decision loses at most $2\alpha_q\varepsilon$ in the linearized workload-conditioned objective. The result is conditional on the calibration bound; current held-out compliance does not establish such a production bound.

This appendix establishes a claim hierarchy rather than one architecture-wide guarantee. The two-expert mixture identity is exact; the centered reduced form is second order; the multi-layer top-$k$ bridge is conditional with explicit corrections; FWP amplification is a within-request covariance condition; and the signed LP score is KKT-consistent with the LP optimum only under optimal prices and primal-feasible tie allocation. The empirical headline remains the offline fixed-population comparison in Table~\ref{tab:routing-fractions}.

\section{Detailed Evaluation Protocol}
\label{app:eval-protocol}

\begin{table*}[t]
\centering
\small
\setlength{\tabcolsep}{5pt}
\caption{Evidence classes used in this paper. Results are compared only within a row unless an explicit bridge is stated.}
\label{tab:headline-evidence}
\begin{tabular}{@{}>{\raggedright\arraybackslash}p{0.19\textwidth}>{\raggedright\arraybackslash}p{0.18\textwidth}>{\raggedright\arraybackslash}p{0.24\textwidth}>{\raggedright\arraybackslash}p{0.31\textwidth}@{}}
\toprule
\textbf{Evidence class} & \textbf{Model and scope} & \textbf{Protocol} & \textbf{Permitted claim} \\
\midrule
Complete-instance quality
& Qwen; all 6,144 expert blocks
& HQQ W2/W3/W4; 88 extended prompts
& Request-level quality distributions and fixed-population allocation inputs. \\
Sampled structural diagnostics
& Qwen and DeepSeek; selected blocks/layers
& LOEO, route agreement, additivity, 65-configuration ranking
& Mechanism and theory-boundary evidence, not complete-instance cross-model quality. \\
Offline population allocation
& Qwen; same 88 requests and complete losses
& \(\tau=0.1513\); analytical factors \(2,4/3,1\)
& Model-based allocation multiplier, not live throughput. \\
Prompt-only FWP
& Qwen; reference W2 score predicting W3 risk
& Cross-instance Spearman/AUROC and prefix sweep
& Calibration feature and abstention boundary, not direct online quality measurement. \\
Endpoint capacity
& Qwen; same two-A100 budget
& Separate serving checkpoints; 216 requests, five repeats
& Endpoint rate vector for that backend, not quality-aligned routed-pool throughput. \\
Live routed prototype
& Three serving GPUs plus one INT3 scorer GPU
& Repeated workload; threshold dispatch
& Feasibility diagnostic only; not an equal-total-resource replica comparison. \\
\bottomrule
\end{tabular}
\end{table*}

\paragraph{Complete-instance quality.}
The extended population contains 32 OpenBookQA prompts, 32 MBPP prompts, and 24 CNN/DailyMail prompts. For each prompt and each W2/W3/W4 instance, average NLL is compared with the reference checkpoint. All 6,144 Qwen expert blocks are quantized; the resulting request rows produce Table~\ref{tab:complete-instance-quality}. The sampled LOEO and 65-configuration studies remain separate because nonsampled experts stay at reference precision in those artifacts.

\paragraph{Fixed-population allocation.}
Every policy in Table~\ref{tab:routing-fractions} receives the same 88 request-level W2/W3/W4 losses, budget \(\tau=0.1513\), and analytical per-request multipliers \(2,4/3,1\). Static and capacity-proportional policies use no request information. The request-agnostic budget-aware mixture chooses one global W3/W4 mixture satisfying the population budget. Sequence length, activation frequency, and FWP supply alternative request orderings. The restricted reference orders requests by measured W3 loss after evaluation and is not deployable. No row models queueing, batching, scorer contention, or reconfiguration.

\paragraph{Prompt-only calibration and compliance.}
Same-instance FWP is reported only as a measurement sanity check. The deployable diagnostic uses the raw W2 prompt FWP and prompt-observable features to predict W3 request loss, with Spearman \(0.586\) and AUROC \(0.747\). Prefix FWP is recomputed at 64, 128, 256, 512, and 1,024 tokens. Calibration, threshold selection, and final evaluation must use disjoint data; the current small repeated 50/50 splits achieve only about \(35\)--\(50\%\) budget-compliant splits and therefore do not satisfy a deployment acceptance gate.

\paragraph{Memory and reconfiguration.}
Memory values in Table~\ref{tab:memory-accounting} are parameter estimates derived from expert and non-expert parameter counts on one accounting basis. They exclude KV cache, activations, runtime workspaces, fragmentation, and scorer residence. The reconfiguration coefficients in Eq.~\eqref{eq:reconfiguration-cost} are formal placeholders until weight-load, eviction, and placement times are measured on the target backend. The DeepSeek-V3 ten-minute redundant-expert example motivates the timescale separation but does not calibrate these coefficients.

\paragraph{Endpoint and overhead protocols.}
The same-two-A100 endpoint test disables prefix caching and repeats 216 requests five times; 128 requests have 1,023 prompt tokens and 128 output tokens. Reported values are mean output tokens/s with population standard deviation. FWP instrumentation is timed within one transformers prefill at 512 and 1,024 tokens. LP times use a CPU synthetic control-window benchmark, and score time covers three candidate comparisons. None is labeled end-to-end routing overhead.

\paragraph{Reference routing procedure.}
Algorithm~\ref{alg:lightweight-routing} states the execution path consistent with the theory and evidence. It is a reference design: the complete online price controller and delayed quality-feedback path have not been deployed in the live prototype.

\begin{figure}[t]
\renewcommand{\figurename}{Algorithm}
\centering
\framebox[\textwidth]{%
\begin{minipage}{0.94\textwidth}
\small
\noindent\textbf{Fixed-pool prompt-risk routing}

\noindent\textbf{Input:} immutable resident pool \(\instances\), reference instance \(a_{\mathrm{ref}}\), calibrated risk maps \(\widehat r_a(z,|I_p|,g)\), service costs \(s_{ga}\) including target re-prefill when required, prices \((\pi,\nu)\), quality budget \(\tau_g\), validated prefix threshold \(T_{\min}\), conservative fallback \(a_{\mathrm{safe}}\)

\vspace{3pt}
\noindent\textbf{Offline and once per provisioning epoch:}
\begin{enumerate}[leftmargin=1.2em, itemsep=0pt, topsep=0pt, label=\arabic*.]
\item Lock \(y^{(q)}\), checkpoint hashes, endpoint capacities, memory footprints, and the reconfiguration record.
\item Split task families into training, calibration, test, and OOD audits; fit \(\widehat r_a\) from raw reference-instance prompt FWP without a same-request full-precision input.
\item Select \(T_{\min}\), uncertainty thresholds, and \(a_{\mathrm{safe}}\) on calibration data only.
\item Solve the window LP or publish a price/target-share snapshot; record convergence and feasibility residuals.
\end{enumerate}

\vspace{3pt}
\noindent\textbf{Online for request \(x\):}
\begin{enumerate}[leftmargin=1.2em, itemsep=0pt, topsep=0pt, label=\arabic*.]
\setcounter{enumi}{4}
\item If \(|I_p(x)|<T_{\min}\) or the request is outside calibration support, send the complete request to \(a_{\mathrm{safe}}\).
\item Otherwise run the prompt prefill on \(a_{\mathrm{ref}}\) and compute \(z(x)=\ell_{\FWP}(x;\bits^{(a_{\mathrm{ref}})})\) over non-padding prompt tokens.
\item For each resident endpoint \(a\), compute \(\losshat_a(x)=\widehat r_a(z(x),|I_p(x)|,g)\) and
\[
S_a=L_g-\pi_a s_{ga}-\nu_g\bigl(\losshat_a(x)-\tau_g\bigr)+\beta d_{ga}.
\]
\item Choose \(a^*\in\arg\max_a S_a\), using the declared quota rule for ties.
\item If \(a^*\ne a_{\mathrm{ref}}\), discard the incompatible reference KV state and run a target prefill on \(a^*\); otherwise reuse the reference KV state.
\item Decode the complete response on \(a^*\) and log service cost, fallback status, and calibration diagnostics.
\end{enumerate}

\vspace{3pt}
\noindent\textbf{Once per control window:}
\begin{enumerate}[leftmargin=1.2em, itemsep=0pt, topsep=0pt, label=\arabic*.]
\setcounter{enumi}{10}
\item Update capacity prices from realized utilization. Update quality prices only from a declared delayed-audit or conservative risk-control signal; true free-running loss is not an immediate online observable.
\end{enumerate}
\end{minipage}}
\caption{Reference procedure for routing within a fixed resident pool. Its data-plane score is \(O(|\instances|)\); reference prefill, possible target re-prefill, and control-plane work are separately charged.}
\label{alg:lightweight-routing}
\end{figure}
\end{document}